\documentclass[journal]{IEEEtran}

\usepackage[vlined,linesnumbered]{algorithm2e}
\usepackage{amsmath}
\usepackage{amsthm}
\usepackage{amssymb}
\usepackage{bbm}
\usepackage{caption}
\usepackage{color}
\usepackage{enumitem}
\usepackage{float}
\usepackage{geometry}
\usepackage{graphicx}
\usepackage[hidelinks]{hyperref}
\usepackage{nameref}
\usepackage[normalem]{ulem}
\usepackage{subcaption}

\newtheorem{theorem}{Theorem}
\newtheorem{definition}{Definition}
\graphicspath{{figures/}}
\mathchardef\dash="2D
\mathchardef\dash="2D

\begin{document}

% paper title
% Titles are generally capitalized except for words such as a, an, and, as,
% at, but, by, for, in, nor, of, on, or, the, to and up, which are usually
% not capitalized unless they are the first or last word of the title.
% Linebreaks \\ can be used within to get better formatting as desired.
% Do not put math or special symbols in the title.
\title{Learning Manipulation Skills Via Hierarchical Spatial Attention}

% author names and IEEE memberships
% note positions of commas and nonbreaking spaces ( ~ ) LaTeX will not break
% a structure at a ~ so this keeps an author's name from being broken across
% two lines.
% use \thanks{} to gain access to the first footnote area
% a separate \thanks must be used for each paragraph as LaTeX2e's \thanks
% was not built to handle multiple paragraphs

\author{Marcus Gualtieri and Robert Platt% <-this % stops a space
\thanks{Khoury College of Computer Sciences, Northeastern University, Boston,
MA, 02115 USA e-mail: mgualti@ccs.neu.edu.}}% <-this % stops a space
%\thanks{Manuscript received April 20, 2019.}}

% The paper headers
%\markboth{Journal of \LaTeX\ Class Files,~Vol.~14, No.~8, August~2015}%
%{Gualtieri and Platt: Hierarchical Spatial Attention}
% The only time the second header will appear is for the odd numbered pages
% after the title page when using the twoside option.
% 
% *** Note that you probably will NOT want to include the author's ***
% *** name in the headers of peer review papers.                   ***
% You can use \ifCLASSOPTIONpeerreview for conditional compilation here if
% you desire.

% make the title area
\maketitle

% As a general rule, do not put math, special symbols or citations in the abstract or keywords.
\begin{abstract}
Learning generalizable skills in robotic manipulation has long been challenging due to real-world sized observation and action spaces. One method for addressing this problem is attention focus -- the robot learns where to attend its sensors and irrelevant details are ignored. However, these methods have largely not caught on due to the difficulty of learning a good attention policy and the added partial observability induced by a narrowed window of focus. This article addresses the first issue by constraining gazes to a spatial hierarchy. For the second issue, we identify a case where the partial observability induced by attention does not prevent Q-learning from finding an optimal policy. We conclude with real-robot experiments on challenging pick-place tasks demonstrating the applicability of the approach.
\end{abstract}

% Note that keywords are not normally used for peerreview papers.
%\begin{IEEEkeywords}
%IEEE, IEEEtran, journal, \LaTeX, paper, template.
%\end{IEEEkeywords}

% For peer review papers, you can put extra information on the cover
% page as needed:
% \ifCLASSOPTIONpeerreview
% \begin{center} \bfseries EDICS Category: 3-BBND \end{center}
% \fi
%
% For peerreview papers, this IEEEtran command inserts a page break and
% creates the second title. It will be ignored for other modes.
%\IEEEpeerreviewmaketitle

% =========================================================================================
\section{Introduction}

Learning robotic manipulation has remained an active and challenging research area. This is because the real-world environments in which robots exist are large, dynamic, and complex. Partial observability -- where the robot does not at once perceive the entire environment -- is common and requires reasoning over past perceptions. Additionally, the ability to generalize to new situations is critical because, in the real world, new objects can appear in different places unexpectedly.

The particular problem addressed in this paper is the large space of possible robot observations and actions -- how the robot processes its past and current perceptions to make high-dimensional decisions. Visual attention has long been suggested as a solution to this problem \cite{Whitehead1991}. Focused perceptions can ignore irrelevant details, and generalization is improved by the elimination of the many irrelevant combinations of object arrangements \cite{Whitehead1991}. Additionally, as we later show, attention can result in a substantial reduction to the number of actions that need considered. Indeed, when selecting position, the number of action choices can become logarithmic rather than linear in the volume of the robot's workspace. In spite of these benefits, visual attention has largely not caught on due to (a) the additional burden of learning where to attend and (b) additional partial observability caused by the narrowed focus.

We address the first challenge -- efficiently learning where to attend -- by constraining the system to a spatial hierarchy of attention. On a high level this means the robot must first see a large part of the scene in low detail, select a position within that observation, and see the next observation in more detail at the position previously selected, and so on for a fixed number of gazes. We address the second challenge -- partial observability induced by the narrowed focus -- by identifying attention with a type of state-abstraction which preserves the ability to learn optimal policies with efficient reinforcement learning (RL) algorithms.

This article extends our prior work \cite{Gualtieri2018B}, wherein we introduced the hierarchical spatial attention (HSA) approach and demonstrated it on 3 challenging, 6-DoF, pick-place tasks. New additions include (a) faster training and inference times, (b) more ablation studies and comparisons to related work, (c) better understanding of when an optimal policy can be learned when using this approach, (d) longer time horizons, and (e) improved real-robot experimental results.

The rest of the paper is organized as follows. First is related work (Section~\ref{sec:relatedWork}). Next, the general manipulation problem is described and the visual attention aspect is added (Sections~\ref{sec:problem} and \ref{sec:senseMoveEffect}). After that, the HSA constraints are added, and this approach is viewed as a generalization of earlier approaches (Section~\ref{sec:hsa} to \ref{sec:implementation}). The bulk of the paper includes analysis and comparisons in 4 domains of increasing complexity (Section~\ref{sec:domains}). Real robot experiments are described close to the end (Sections~\ref{sec:bottlesOnCoasters} and \ref{sec:pickPlace}). Finally, we conclude with what we learned and future directions (Section~\ref{sec:conclusion}).

% =========================================================================================
\section{Related Work}
\label{sec:relatedWork}

% Summary of the 3 areas. Our work builds off of deictic image mapping and DQN.

This work is most related to robotic manipulation, reinforcement learning, and attention models. It is extends our prior research on 6-DoF pick-place \cite{Gualtieri2018B} and primarily builds on DQN \cite{Mnih2015} and Deictic Image Mapping \cite{Platt2019}.

% -----------------------------------------------------------------------------------------
\subsection{Learning Robotic Manipulation}
\label{sec:learningRoboticManipulation}

Traditional approaches to robotic manipulation consider known objects -- a model of every object to be manipulated is provided in advance \cite{LozanoPerez1986,Tournassoud1987,Tremblay2018}. While these systems can be quite robust in controlled environments, they encounter failures when the shapes of the objects differ from expected. Recent work has demonstrated grasping of novel objects by employing techniques intended to address the problem of generalization in machine learning \cite{Lenz2015,Pinto2016,Levine2016A,Mahler2017,tenPas2017,Kalashnikov2018,Morrison2018,Quillen2018,Zeng2018}.

There have been attempts to extend novel object grasping to more complex tasks such as pick-place. However, these have assumed either fixed grasp choices \cite{Jiang2012} or fixed place choices \cite{Gualtieri2018A}. The objective of the present work is to generalize these attempts -- a single system that can find 6-DoF grasp and place poses.

Other research considers grasping and pushing novel objects to a target location \cite{Xie2018}. Their approach is quite different: a predictive model of the environment is learned and used for planning, whereas we aim to learn a policy directly. Other work has considered the problem of domain transfer \cite{James2017} and sparse rewards in RL \cite{Andrychowicz2017}. We view these as complimentary ideas that could be combined with our approach for an improvement.

% Watch the kPAM paper.

% -----------------------------------------------------------------------------------------
\subsection{Reinforcement Learning}

Like several others, we apply RL techniques to the problem of robotic manipulation (see above-mentioned \cite{Levine2016A,Kalashnikov2018,Quillen2018,Gualtieri2018A,Andrychowicz2017} and survey \cite{Kober2013}). RL is appealing for robotic control for several reasons. First, several algorithms (e.g., \cite{Watkins1989,Rummery1994}) do not require a complete model of the environment. This is of particular relevance to robotics, where the environment is dynamic and difficult to describe exactly. Additionally, observations are often encoded as camera or depth sensor images. Deep Q-Networks (DQN) demonstrated an agent learning difficult tasks (Atari games) where observations were image sequences and actions were discrete \cite{Mnih2015}. An alternative to DQN that can handle continuous action spaces are actor-critic methods like DDPG \cite{Lillicrap2015}. Finally, RL -- which has its roots in optimal control -- provides tools for the analysis of learning optimal behavior (e.g. \cite{Watkins1992,Jaakkola1994,Li2006}), which we refer to in Section~\ref{sec:tabularDomain}.

% -----------------------------------------------------------------------------------------
\subsection{Attention Models}

Our approach is inspired by models of visual attention. Following the early work of Whitehead and Ballard \cite{Whitehead1991}, we distinguish overt actions (which directly affect change to the environment) from perceptual actions (which retrieve information). Similar to their agent model, our abstract robot has a virtual sensor which can be used to focus attention on task-relevant parts of the scene. The present work updates their methodology to address more realistic problems, and we extend their analysis by describing a situation where an optimal policy can be learned even in the presence of ``perceptual aliasing'' (i.e. partial observability).

Attention mechanisms have also been used with artificial neural networks to identify an object of interest in a 2D image \cite{Sprague2004,Larochelle2010,Mnih2014,Jaderberg2015}. Our situation is more complex in that we identify 6-DoF poses of the robot's hand. Improved grasp performance has been observed by active control of the robot's sensor \cite{Gualtieri2017,Morrison2019}. These methods attempt to identify the best sensor placement for grasp success. In contrast, our robot learns to control a virtual sensor for the purpose of reducing the complexity of action selection and learning.

Work contemporary with ours also considered attention for controlling high-dimensional manipulators \cite{Wu2019}. Important differences from our approach include the use of policy gradient instead of value-based methods, sensing 2D depth instead of 3D point clouds, and learned instead of fixed attention parameters.

% =========================================================================================
\section{Problem Statement}
\label{sec:problem}

The problem considered herein is learning to control a move-effect system (Fig.~\ref{fig:moveEffect}, cf. \cite{Platt2019}):

\begin{definition}[Move-Effect System]
A \emph{move-effect system} is a discrete time system consisting of a robot, equipped with a depth sensor and end effector (e.e.), and rigid objects of various shapes and configurations. The robot perceives a history of point clouds $C_1, \dots, C_k$, where $C_i \in \mathbb{R}^{n_c\times3}$ is acquired by the depth sensor; an e.e. status, $h \in \{ 1, \dots, n_\mathit{h}\}$; and a reward $r \in \mathbb{R}$. The robot's action is $\mathit{move\dash effect(T_\mathit{ee}, o)}$, where $T_\mathit{ee} \in W$ is the pose of the e.e., $W \subseteq \mathit{SE}(3)$ is the robot's workspace, and $o \in \{1, \dots, n_o\}$ is a preprogrammed controller for the e.e. For each stage $t=1,\dots,t_\mathit{max}$, the robot receives a new perception and takes an action.
\end{definition}

\noindent The reward is usually instrumented by the system engineer to indicate progress toward completion of some desired task. The robot initially has no knowledge of the system's state transition dynamics. The objective is, by experiencing a sequence of episodes, for the robot to learn a policy -- a mapping from observations to actions -- which maximizes the expected sum of per-episode rewards.

\begin{figure}[ht]
  \centering
  \includegraphics[width=0.70\linewidth]{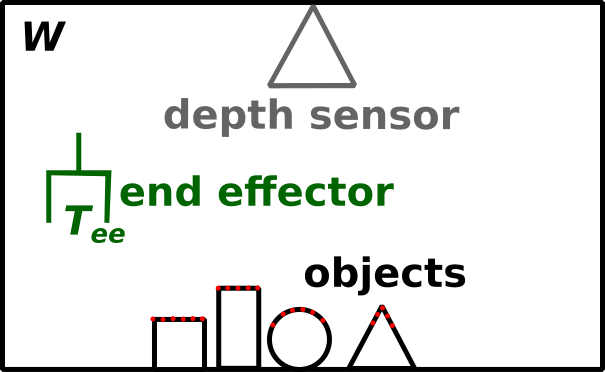}
  \caption{The move-effect system. The robot has an e.e. which can be moved to pose $T_\mathit{ee}$ to perform operation $o$.}
  \label{fig:moveEffect}
\end{figure}

For example, suppose the e.e. is a 2-fingered gripper, $o \in \{\mathit{open}, \mathit{close} \}$, $h \in \{\mathit{empty}, \mathit{holding} \}$, the objects are bottles and coasters, and the task is to place all the bottles on the coasters. The reward could be $1$ for placing a bottle on a coaster, $-1$ for removing a placed bottle, and $0$ otherwise.

% =========================================================================================
\section{Approach}
\label{sec:approach}

Our approach has two parts. The first part is to reformulate the problem as a Markov decision process (MDP) with abstract states and actions (Section~\ref{sec:senseMoveEffect}). With this reformulation, the resulting state representation is substantially reduced, and it becomes possible for the robot to learn to restrict attention to task-relevant parts of the scene. The second part is to add constraints to the actions so that e.e. pose is decided sequentially (Section~\ref{sec:hsa}). After these improvements, the problem is then amenable to solution via standard RL algorithms like DQN.

% -----------------------------------------------------------------------------------------
\subsection{Sense-Move-Effect MDP}
\label{sec:senseMoveEffect}

The sense-move-effect system adds a controllable, virtual sensor which perceives a portion of the point cloud from a parameterizable perspective (Fig.~\ref{fig:senseMoveEffect}). 

\begin{definition}[Sense-Move-Effect System]
A \emph{sense-move-effect system} is a move-effect system where the robot's actions are augmented with $\mathit{sense}(T_s, z)$ (where $T_s \in W$ and $z \in \mathbb{R}_{>0}^3$) and the point cloud observations $C_1, \dots, C_k$ are replaced with a history of $k$ images, $I_1, \dots, I_k$ (where $I \in \mathbb{R}^{n_\mathit{ch} \times n_x \times n_y}$). The $\mathit{sense}$ action has the effect of adding $I = \mathit{Proj}(\mathit{Crop}(\mathit{Trans}(T_s, C_k), z))$ to the history.% <- suppress space
% <- suppress new line
\footnote{$\mathit{Proj}: \mathbb{R}^{n_c \times 3} \to \mathbb{R}^{n_\mathit{ch}\times n_x\times n_y}$ is $n_\mathit{ch}$ orthographic projections of points onto $n_\mathit{ch}$, $n_x \times n_y$ images. Each image plane is positioned at the origin with a different orientation. Image values are the point to plane distance, ambiguities resolved with the nearest distance. $\mathit{Crop}: \mathbb{R}^{n_c \times 3} \to  \mathbb{R}^{n_{c'} \times 3}$ returns the $n_{c'} \leq n_c$ points of $C$ which lie inside a rectangular volume situated at the origin with length, width, height $z$. $\mathit{Trans}(T_s, C)$ expresses $C$ (initially expressed w.rt. the world frame) w.r.t. $T_s$.}
\end{definition}

\begin{figure}[ht]
  \centering
  \includegraphics[width=0.70\linewidth]{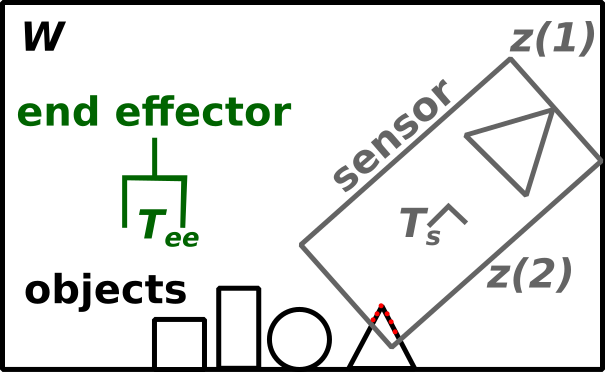}
  \caption{The sense-move-effect system adds a virtual, mobile sensor which observes points in a rectangular volume at pose $T_s$ with size $z$.}
  \label{fig:senseMoveEffect}
\end{figure}

The $\mathit{sense}$ action makes it possible for the robot to get either a compact overview of the scene or to attend to a small part of the scene in detail. Since the resolution of the images is fixed, large values of $z$ correspond to seeing more objects in less detail, and small values of $z$ correspond to seeing less objects in more detail.

The robot's memory need not include the last $k$ images -- it can include any previous $k$ images selected according to a predetermined strategy. Because the environment only changes after $\mathit{move\dash effect}$ actions, we keep the latest image, $I_k$, and the last $k-1$ images that appeared just before $\mathit{move\dash effect}$ actions. Fig.~\ref{fig:stateImages} shows an example in the bottles on coasters domain.

\begin{figure}[ht]
  \centering
  \includegraphics[width=0.23\linewidth]{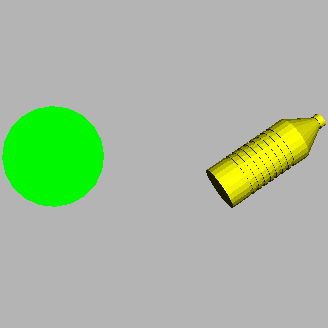}
  \includegraphics[width=0.23\linewidth]{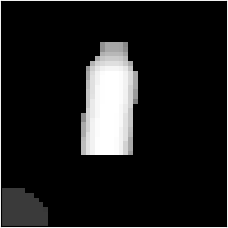}
  \includegraphics[width=0.23\linewidth]{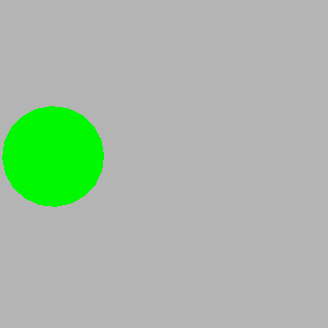}
  \includegraphics[width=0.23\linewidth]{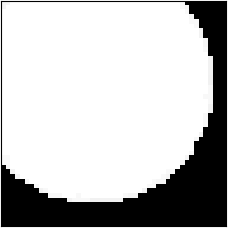}
  \caption{Scene and observed images for $k=2$ and $n_\mathit{ch}=1$. \textbf{Left}. Scene's initial appearance. \textbf{Left center}. $\mathit{sense}$ image (large $z$) just before $\mathit{move\dash effect}(T_\mathit{ee}, \mathit{close})$. \textbf{Right center}. Scene's current appearance. \textbf{Right}. Current $\mathit{sense}$ image, focused on the coaster (small $z$).}
  \label{fig:stateImages}
\end{figure}

In order to apply standard RL algorithms to the problem of learning to control a sense-move-effect system, we define the sense-move-effect MDP.

\begin{definition}[Sense-Move-Effect MDP]
Given a sense-move-effect system, a reward function, and transition dynamics, a \emph{sense-move-effect MDP} is a finite horizon MDP where states are sense-move-effect system observations and actions are sense-move-effect system actions.
\end{definition}

\noindent The reward function and transition details are task and domain specific, respectively, examples of which are given in Section~\ref{sec:domains}.

% -----------------------------------------------------------------------------------------
\subsection{Hierarchical Spatial Attention}
\label{sec:hsa}

The observation is now similar to that of DQN -- a short history of images plus the e.e. status -- and can be used by a Q-network to approximate Q-values. However, the action space remains large due to the 6-DoF choice for $T_\mathit{s}$ or $T_\mathit{ee}$ and the 3-DoF choice for $z$. Additionally, it may take a long time for the robot to learn which $\mathit{sense}$ actions result in useful observations. To remedy both issues, we design constraints to the sense-move-effect actions.

\begin{definition}[Hierarchical Spatial Attention]
Given a sense-move-effect system, $L \in \mathbb{N}_{>0}$, $T_s^1 \in W$, and the list of pairs $[(z_1, d_1), \dots, (z_L, d_L)]$, (where $z_i \in \mathbb{R}^3_{>0}$ and $d_i \in \mathbb{R}^6$), \emph{hierarchical spatial attention (HSA)} constrains the robot to take $L$ $\mathit{sense}(T_s, z)$ actions, with $z=z_i$ for $i = 1, \dots, L$, prior to each $\mathit{move\dash effect}$ action. Furthermore, the first sensor pose in this sequence must be $T_s^1$; the sensor poses $T_s^{i+1}$, for $i = 1, \dots, L-1$, must be offset no more than $d_i$ from $T_s^i$; and e.e. pose $T_\mathit{ee}$ must be offset no more than $d_L$ of $T_s^L$.%
\footnote{Concretely, $d_i = [x, y, z, \theta, \phi, \rho]$ indicates a position offset of $\pm x/2$, $\pm y/2$, and $\pm z/2$ and a rotation offset of $\pm \theta/2$, $\pm \phi/2$, and $\pm \rho/2$.}
\end{definition}

The process is thus divided into $t_\mathit{max}$ \textit{overt stages}, where, for each stage, $L$ $\mathit{sense}$ actions are followed by 1 $\mathit{move\dash effect}$ action (Fig.~\ref{fig:process}). The constraints should be set such that the observation size $z_i$ and offset $d_i$ decrease as $i$ increases, so the point cloud under observation decreases in size, and the volume within which the e.e. pose can be selected is also decreasing. These constraints are called hierarchical spatial attention because the robot is forced to learn to attend to a small part of the scene (e.g., Fig.~\ref{fig:hsaExample}).

\begin{figure}[ht]
  \centering
  \includegraphics[width=\linewidth]{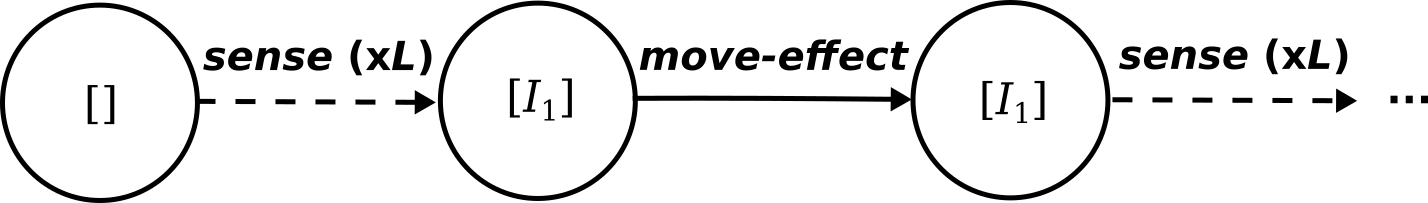}
  \caption{Initially, the state is empty. Then, $L$ sense actions are taken, at each point the latest image is state. After this, the robot takes 1 $\mathit{move\dash effect}$ action. Then, the process repeats, but with the last image before $\mathit{move\dash effect}$ saved to memory.}
  \label{fig:process}
\end{figure}

\begin{figure}[ht]
  \centering
  \includegraphics[width=0.23\linewidth]{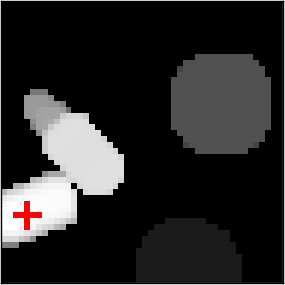}
  \includegraphics[width=0.23\linewidth]{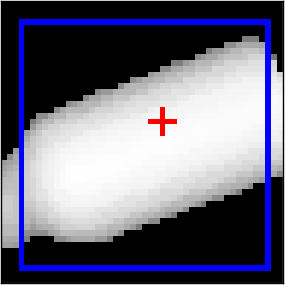}
  \includegraphics[width=0.23\linewidth]{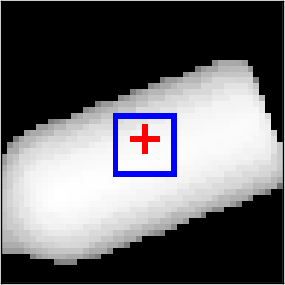}
  \includegraphics[width=0.23\linewidth]{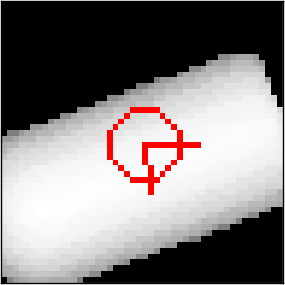}
  \caption{HSA applied to grasping in the bottles on coasters domain (Section~\ref{sec:bottlesOnCoasters}). There are 4 levels (i.e. $L=4$). The sensor's volume size $z$ is $36 \times 36 \times 23.75$ cm for level 1, $10.5 \times 10.5 \times 47.5$ cm for levels 2 and 3, and $9 \times 9 \times 47.5$ cm for level 4. As indicated by blue squares, $d$ constrains position in the range of $\pm 18 \times 18 \times 6.875$ cm for level 1, $\pm 4.5$ cm$^3$ for level 2, and $\pm 1.125$ cm$^3$ for level 3. Orientation is selected for level 4 in the range of $\pm 90^\circ$ about the hand approach axis. Red crosses indicate the $x,y$ position selected by the robot, and the red circle indicates the angle selected by the robot. Positions are sampled uniformly on a $6\times6\times6$ grid and 60 orientations are uniformly sampled. Pixel values normalized and height selection not shown for improved visualization.}
  \label{fig:hsaExample}
\end{figure}

To see how HSA can improve action sample efficiency, consider the problem of selecting position in a 3D volume. Let $\alpha$ be the largest volume allowed per sample. With naive sampling, the required number of samples $n_s$ is proportional to the workspace volume $v_0 = d_1(1) d_1(2) d_1(3)$, i.e., $n_s = \lceil v_0/\alpha \rceil$. But with HSA, we select position sequentially, by say, halving the volume size in each direction at each step, i.e., $d_{i+1} = 0.5 d_i$. In this case $8L$ samples are needed, i.e., a sample for each octant at each step. The volume represented by each sample at step $i$, for $i=1,\dots,L$, is $v_i = v_0 / 8^i$. To get $v_L \leq \alpha$, i.e., to get the volume represented by samples used for selecting e.e. position to be no more than $\alpha$, $L = \lceil \text{log}_8 (v_0 / \alpha) \rceil$. Thus, with HSA, the sample complexity becomes logarithmic, rather than linear, in $v_0$.

% -----------------------------------------------------------------------------------------
\subsection{Lookahead Sense-Move-Effect}
\label{sec:lookahead}

So far we have not specified how action parameters $T_s$, $T_\mathit{ee}$, and $z$ are encoded. For \textit{standard sense-move-effect}, these are typically encoded as 6 floating point numbers representing the pose $T$ and 3 floating point numbers representing the volume size $z$. Alternatively, the pair $(T, z)$ could be encoded as the $\textit{sense}$ image that \emph{would} be seen if the sensor were to move to pose $T$ with zoom $z$. This is as if the action was ``looking ahead'' at the pose the sensor or e.e. would move to if this action were selected.

In particular, the \textit{lookahead sense-move-effect} MDP has actions $\mathit{sense}(T_s, z_s)$ and $\mathit{move\dash effect}(T_\mathit{ee}, z_\mathit{ee}, o)$, the difference being the additional parameter $z_\mathit{ee} \in \mathbb{R}^3_{>0}$ for $\mathit{move\dash effect}$. The action samples are encoded as the height map that would be generated by $\mathit{sense}(T, z)$. Because action has this rich encoding, state is just the e.e. status and a history of $k$ actions.

The HSA constraints for the lookahead variant have the same parameterization -- an initial pose $T_s^1$ and a list of pairs $\left[ (z_1, d_1), \dots, (z_L, d_L) \right]$. The semantics are slightly different. $z_i$ for $i = 1, \dots, L-1$ is the $z_s$ parameter for the $i$th $\mathit{sense}$, and $z_L$ is the $z_\mathit{ee}$ parameter. The $d_i$ for $i = 1, \dots, L-1$ specify the offset of the $\mathit{sense}$ action samples relative to the last pose decided, $T_s^i$. $d_L$ specifies the offset of $T_\mathit{ee}$ relative to $T_s^{L}$.

% -----------------------------------------------------------------------------------------
\subsection{Relation to Other Approaches in the Literature}
\label{sec:otherApproaches}

\subsubsection{DQN}
Consider a sense-move-effect MDP with HSA constraints $L=1$, $T_s^1$ centered in the robot's workspace, and $z_1$ and $d_1$ large enough to capture the entire workspace. The only free action parameters for this system are the e.e. pose, which is sampled uniformly and spaced appropriately for the task, and the e.e. operation. In this case, the observations and actions are similar to that of DQN \cite{Mnih2015}, and the DQN algorithm can be applied to the resulting MDP.

However, this approach is problematic in robotics because the required number of action samples is large, and the image resolution would need to be high in order to capture the required details of the scene. For example, a pick-place task where e.e. poses are in $\mathit{SE}(3)$, the robot workspace is 1 m$^3$, the required position precision is 1 mm, and the required orientation resolution is 1$^\circ$ per Euler angle requires on the order of $10^{17}$ actions. Adding more levels (i.e. $L>1$) alleviates this problem.

\subsubsection{Deictic Image Mapping}
With $L=1$, $T_s^1$ centered in the robot's workspace, $z_1$ the deictic marker size (e.g., the size of the largest object to be manipulated), and $d_1$ large enough to capture the entire workspace, HSA applied to the lookahead sense-move-effect MDP is the Deictic Image Mapping representation \cite{Platt2019}. Similar to the case with DQN, if the space of e.e. poses is large, and precise positioning is needed, many actions need to be sampled. In fact, the computational burden with the Deictic Image Mapping representation is even larger than that of DQN due to the need to create images for each action. Yet, the deictic representation has significant advantages over DQN in terms of efficient learning due to its small observations \cite{Platt2019}.

HSA generalizes and improves upon both DQN and Deictic Image Mapping by overcoming the burden for the agent to select from many actions in a single time step. Instead, the agent sequentially refines its choice of e.e. pose over a sequence of $L$ decisions. We provide comparisons between these approaches in Section~\ref{sec:domains}.

% -----------------------------------------------------------------------------------------
\subsection{Implementation Methods}
\label{sec:implementation}

To implement HSA for a sense-move-effect MDP, it is necessary to select values for HSA parameters and a training algorithm. Here we provide rough guidelines for making both choices for standard HSA.

\subsubsection{HSA Parameter Values}
\label{sec:parameters}

Ideal values for $T_s^1$, $L$, and $[(z_1, d_1), \dots (z_L, d_L)]$ depend on the position and size of the robot's workspace, the desired e.e. precision, and available computing resources. In our implementations, we have separate levels for selecting position and orientation, with position selecting levels occurring first. The procedure for deciding position selecting levels is as follows. First, the position component of the initial sensor pose $T_s^1$ is set to the center of the robot's workspace. Second, the number of action samples $n_s$ depends on computing resources, e.g., the number of Q-values that can be evaluated in parallel. If $n_s=n^3$, where $n$ is the number of position samples spaced evenly along an axis, then $n$ is set to the largest integer such that $n_s$ samples can be evaluated efficiently. Third, the number of levels $L$ is the minimum number of times the workspace needs divided to achieve the desired e.e. precision. If $p \in \mathbb{R}_{>0}^3$ is the desired e.e. precision and $w \in \mathbb{R}_{>0}^3$ is the size of the workspace, $L=\max_{i=1,\dots,3}\lceil\log_n(w(i)/p(i))\rceil$. Fourth, sampling regions $d_i$ for $i=1,\dots,L$ should be large enough so that, if patches size $d_i$ are centered on samples in level $i-1$, the entire region is covered: $d_i=w/n^{i-1}$. Lastly, observation sizes $z_i$ for $i=1,\dots,L$ should be equal to $d_i$ or the size of the largest object to be manipulated, whichever is largest. The latter condition is necessary if the entire object must be visible to determine the appropriate action. For example, when grasping bottles to be placed upright, either the top or bottom of the bottle must be visible to determine bottle orientation in the hand. Deciding orientation selecting levels is simpler: add 1 level per Euler angle, each with the desired angular e.e. precision.

% Is there a reason to have redundant samples? E.g., to avoid having to reason about objects near the edge of the image?

\subsubsection{Training Algorithm}
\label{sec:trainingAlgorithm}

\begin{algorithm}[ht]
\DontPrintSemicolon
\caption{Train standard HSA.}
\label{alg:training}
\SetKwInOut{input}{Input}
\input{$\mathit{nEpisodes}$, $t_\mathit{max}$, $T_1$, $n_s$, $L$, $[(z_1, d_1), \dots, (z_L, d_L)]$, $\mathit{maxExperiences}$, $\mathit{trainEvery}$}
Initialize $Q$, $D$, $\epsilon$\;
\For {$i \gets 1, ..., \mathit{nEpisodes}$}
{
  $\mathit{env} \gets \mathit{initialize\dash environment}(i)$\;
  
  \For {$t \gets 1, \dots, t_{\mathit{max}}$}
  {
    $h = \mathit{get\dash ee\dash status}(env)$\;
    
    $I_h = \textsc{null}$\;
    \If {$t > 1 \land h = \mathit{holding}$}
    {
      $I_h \gets I$
    }
    
    \For {$l \gets 1, \dots, L$}
    {
      $I \gets \mathit{sense}(T_l, z_l)$\;
      $o' \gets (h, I_h, I)$\;
      $T_{l+1} \gets \mathit{get\dash pose}(Q, o', T_l, d_l, n_s, \epsilon)$\;
      $a' \gets T_l^{-1}T_{l+1}$\;
      
      \If {$t > 1 \lor l > 1$}
      {
        $D \gets D \cup \{(o, a, o', r)\}$\;
      }
      
      $o \gets o'$; $a \gets a'$; $r \gets 0$\;
    }
    
    $\mathit{op} \gets \mathit{get\dash ee\dash op}(h)$\;
    $\mathit{overtAct} \gets \mathit{move\dash effect}(T_{L+1}, op)$\;
    $r \gets \mathit{transition}(env, overtAct)$\;
  }
  
  $D \gets D \cup \{(o, a, \textsc{null}, r)\}$\;
  
  \If {$\mathit{modulo}(i, \mathit{trainEvery}) = 0$}
  {
    $D \gets \mathit{prune\dash exp}(D, \mathit{maxExperiences})$\;
    $Q \gets \mathit{update\dash q\dash function}(D, Q)$\;
  }
  $\epsilon \gets \mathit{decrease\dash epsilon}(|D|)$\;
}
\end{algorithm}

Algorithm~\ref{alg:training} is a variant of DQN \cite{Mnih2015} that follows the HSA constraints. For concreteness, this implementation stores experiences for Q-learning; modification for other temporal difference (TD) update rules, such as Sarsa \cite{Rummery1994} or Monte Carlo (MC) \cite{Sutton2018}, is straight-forward. For simplicity of exposition, we also restrict to the case where image history consists of the current image $I$ and the image $I_h$ before the last grasp, e.e. status is binary $\mathit{empty}$ or $\mathit{holding}$, and the e.e. operation is binary $\mathit{open}$ or $\mathit{close}$. 

Initially, the Q-function gets random weights, the experience replay database is empty, and the probability of taking random actions $\epsilon=1$ (line 1). The environment is initialized to a scene unique to each episode (line 2). For each time step, the e.e. status is observed (line 5), and $I_h$ is the previously observed image if the e.e. is holding something (lines 6-8). Then, for each HSA level, a $\mathit{sense}$ action is taken (line 10), the pose of the next $\mathit{sense}$ action is determined either randomly or according to $Q$ (line 12), and the experience is saved (line 15). Actions are encoded relative to the previous sense pose (line 13). Next, the robot moves the e.e. to $T_{L+1}$ and performs an operation $\mathit{op}$, after which a reward is observed (lines 17 - 19). Finally, after $\mathit{trainEvery}$ episodes, the $Q$ function is updated with the current experiences (lines 21 - 23), and $\epsilon$ is set inversely proportional to the number of experiences (line 24).

% =========================================================================================
\section{Application Domains}
\label{sec:domains}

In this section we compare the HSA approach in 4 application domains of increasing complexity. The complexity increases in terms of the size of the action space and in terms of the diversity of object poses and geometries. We analyze simpler domains because the results are more interpretable and learning is faster (Table~\ref{tab:simulationTimes}). More complex domains are included to demonstrate the practicality of the approach. All training is in simulation, and Sections~\ref{sec:bottlesOnCoasters} and~\ref{sec:pickPlace} include test results for a psychical robotic system. Source code for reproducing the simulated experiments is available at \cite{Gualtieri2019}.

\begin{table*}[ht]
  \centering
  \begin{tabular}{|c|c|c|c|c|c|}
  \hline
  & Tabular Pegs on Disks & Upright Pegs on Disks & Bottles on Coasters & 6-Dof Pick-Place\\
  \hline
  Time (hours) & 0.23 & 1.29 & 8.12 & 96.54\\
  %\hline
  %$n$ Simulations & 30 & 10 & 10 & 15\\
  \hline
  \end{tabular}
  \caption{Average simulation time for the 4 test domains. Times are averaged over 10 or more simulations on 4 different workstations, each equipped with an Intel Core i7 processor and an NVIDIA GTX 1080 graphics card.}
  \label{tab:simulationTimes}
\end{table*}

% -----------------------------------------------------------------------------------------
\subsection{Tabular Pegs on Disks}
\label{sec:tabularDomain}

Here we analyze the HSA approach applied to a simple, tabular domain, where the number of states and actions is finite. The domain consists of 2 types of objects -- pegs and disks -- which are situated on a 3D grid (Fig.~\ref{fig:tabularPegsOnDisks}). The robot can move its e.e. to a location on the grid and open/close its gripper. The goal is for the robot to place all the pegs onto disks.

\begin{figure}[ht]
  \centering
  \includegraphics[width=0.50\linewidth]{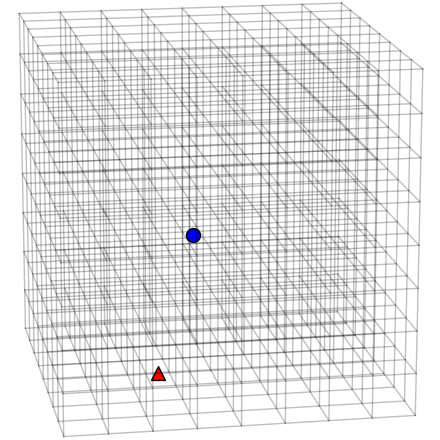}
  \caption{Tabular pegs on disks with an $8\times 8 \times 8$ grid, 1 peg (red triangle), and 1 disk (blue circle).}
  \label{fig:tabularPegsOnDisks}
\end{figure}

If this problem is described as a finite MDP, eventual convergence to the optimal policy is guaranteed for standard RL algorithms \cite{Watkins1992,Jaakkola1994}. However, the number of state-action pairs is too large for practical implementation unless some abstraction is applied. The main question addressed here is if convergence guarantees are maintained with the HSA abstraction.

\subsubsection{Ground MDP}
\label{sec:tabularGroundMdp}

Tabular pegs on disks is first described without the sense-move-effect abstraction.
\begin{itemize}[leftmargin=*]

\item \textbf{State.} A set of pegs $P = \{p_1, \dots, p_n\}$, a set of disks $D = \{d_1, \dots, d_n\}$, and the current time $t \in \{1, \dots, t_\mathit{max}\}$. A peg (resp. disk) is a location $(x,y,z) \in \{1, \dots, m\}^3$ except peg locations are augmented with a special in-hand location $h$. Pegs (resp. disks) cannot occupy the same location at the same time, but 1 peg and 1 disk can occupy the same location at the same time.

\item \textbf{Action.} $\mathit{move\dash effect}(x, y, z)$, which moves the e.e. to $(x,y,z) \in \{1, \dots, m\}^3$ and opens/closes. It opens if a peg is located at $h$ and closes otherwise.

\item\textbf{Transition.} $t$ increments by 1. If no peg is at $h$ and a peg $p$ is at the action location, then the peg is grasped ($p = h$). If a peg is located at $h$ and the action location $a$ does not contain a peg, the peg is placed ($p = a$). Otherwise, the state remains unchanged.

\item \textbf{Reward.} 1 if a peg is placed on an unoccupied disk, -1 if a placed peg is removed, and 0 otherwise.

\end{itemize}

Initially, pegs and disks are at distinct locations, and no peg is in the e.e. The time horizon is $t_\text{max} = 2n$, where there is enough time to grasp and place each peg. This MDP satisfies the Markov property because the next state is completely determined from the current state and action. The number of possible states is shown in Eq.~\ref{eq:tabularGroundMdp}, and the number of actions is $|A| = m^3$. It is not practical to learn the optimal policy by enumerating all state-action pairs for this MDP: for example, if $m=16$ and $n=3$, the state-action value lookup table size is on the order of $10^{24}$.

\begin{equation}
\label{eq:tabularGroundMdp}
|S| = {m^3 + 1\choose n}{m^3 \choose n}t_\mathit{max}
\end{equation}

\subsubsection{Sense-Move-Effect MDP}
\label{sec:tabularAbstractMdp}

We apply the sense-move-effect abstraction of Section~\ref{sec:senseMoveEffect} and HSA constraints of Section~\ref{sec:hsa} to the tabular pegs on disks problem. The process is illustrated in Fig.~\ref{fig:tabularHsa}. At level 1, the sensor perceives the entire $m^3$ grid as 8 cells, each cell summarizing the contents of an octant of space in the underlying grid. The robot then selects one of these cells to attend to. At levels $2, \dots, L-1$, the sensor perceives 8 cells revealing more detail of the octant selected in the previous level. At level $L$, the sensor perceives 8 cells in the underlying grid, and the location of the underlying action is determined by the cell selected here. Without loss of generality, assume the grid size $m$ of the ground MDP is a power of 2 and the number of levels $L$ is $\log_2(m)$.

\begin{figure}[ht]
  \centering
  \includegraphics[width=0.32\linewidth]{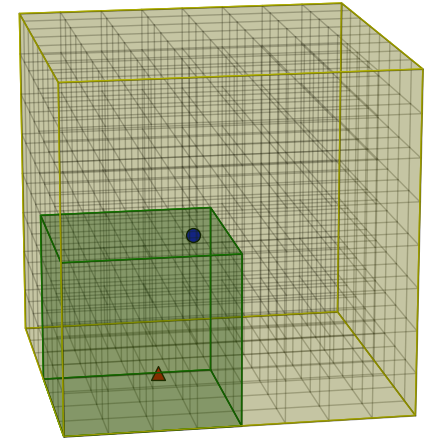}
  \includegraphics[width=0.32\linewidth]{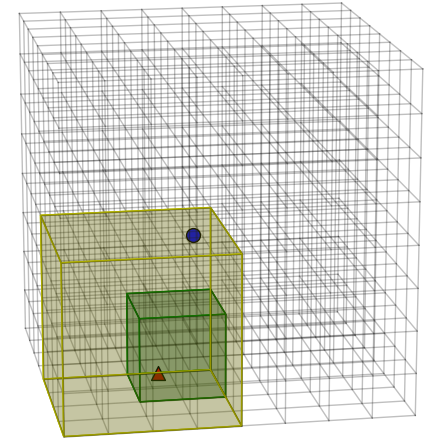}
  \includegraphics[width=0.32\linewidth]{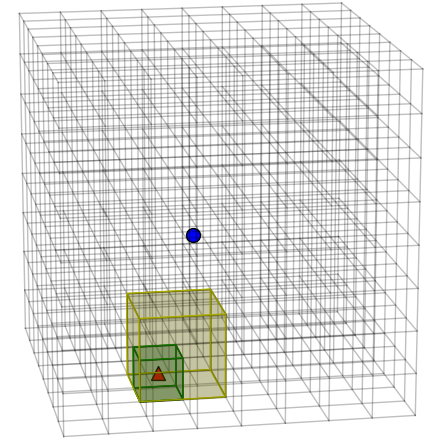}
  \includegraphics[width=0.32\linewidth]{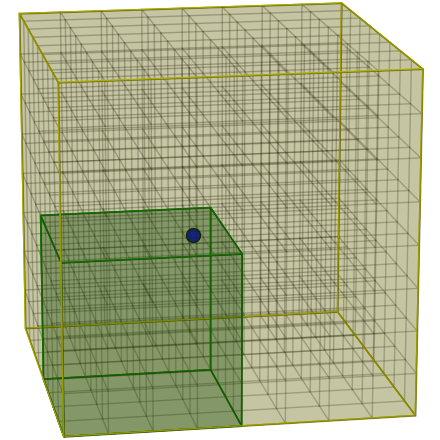}
  \includegraphics[width=0.32\linewidth]{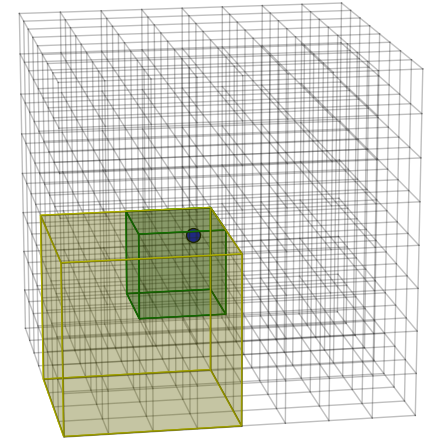}
  \includegraphics[width=0.32\linewidth]{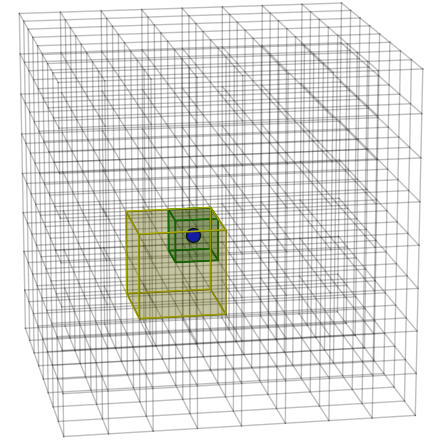}
  \caption{HSA applied to the grid in Fig.~\ref{fig:tabularPegsOnDisks}. \textbf{Columns} correspond to levels 1, 2, and 3. The observed volume appears yellow, and the octant selected by the robot appears green. \textbf{Top}. Robot selects the peg and is holding it afterward. \textbf{Bottom}. Robot selects the disk.}
  \label{fig:tabularHsa}
\end{figure}

\begin{itemize}[leftmargin=*]

\item \textbf{State.} The current level $l \in \{1, \dots, L\}$, the overt time step $t \in \{1, \dots, t_\mathit{max}\}$, a bit $h \in \{0, 1\}$ indicating if a peg is held, and the tuple $G = (G_p, G_d, G_\mathit{pd}, G_{e})$ where each $G_i \in \{0,1\}^8$. $G_p$ indicates the presence of unplaced pegs, $G_d$ unoccupied disks, $G_\mathit{pd}$ placed pegs, and $G_e$ empty space.

\item \textbf{Action.} The action is $a \in \{ 1, \dots, 8\}$, a location in the observed grids.

\item \textbf{Transition.} For levels $l=1, \dots, L-1$, the robot selects a cell in $G$ which corresponds to some partition of space in the underlying grid. The sensor perceives this part of the underlying grid and generates the observation at level $l+1$. For level $L$, the $L$ selections determine the location of the underlying $\mathit{move\dash effect}$ action, $l$ is reset to 1, and otherwise the transition is the same as in the ground MDP.

\item \textbf{Reward.} The reward is 0 for levels $1, \dots, L-1$. Otherwise, the reward is the same as for the ground MDP.

\end{itemize}

The above process is no longer Markov because a history of states and actions could be used to better predict the next state. For instance, for a sufficiently long random walk, the exact location of all pegs and disks could be determined from the history of observations, and the underlying grid could be reconstructed.

On the other hand, this abstraction results in substantial savings in terms of the number of states (Eq.~\ref{eq:tabularSenseMoveEffectMdp}) and actions ($|A|=8$). The only nonconstant term (besides $t_\mathit{max}$) is logarithmic in $m$. Referring to the earlier example with $m=16$ and $n=3$, the state-action lookup table size is the order of $10^{11}$.

\begin{equation}
\label{eq:tabularSenseMoveEffectMdp}
|S| \leq 2^{33} \log_2(m) t_\mathit{max} 
%|S| = 4n \log_2(m) \left[\sum_{i=0}^{\min(8,n)} {8 \choose i}\right]^3 \left[\sum_{i=0}^{\min(8,n)} {8 \choose 8-i}\right]
\end{equation}

\subsubsection{Theoretical Results}
\label{sec:tabularTheoreticalResults}

The sense-move-effect MDP with HSA constraints can be classified according to the state abstraction ordering defined in Li et al. \cite{Li2006}. In particular, we show $Q^*$-irrelevance, which is sufficient for the convergence of a number of RL algorithms, including Q-learning, to a policy optimal in the ground MDP.

\begin{definition}[$Q^*$-irrelevance Abstraction \cite{Li2006}]
Given an MDP $M = \langle S, A, P, R, \gamma \rangle$, any states $s_1, s_2 \in S$, and an arbitrary but fixed weighting function $w(s)$, a $Q^*$\emph{-irrelevance abstraction}  $\phi_{Q^*}$ is such that for any action $a$, $\phi_{Q^*}(s_1) = \phi_{Q^*}(s_2)$ implies $Q^*(s_1, a) = Q^*(s_2, a)$.
\end{definition}

\noindent $\phi_{Q^*}$ is a mapping from ground states to abstract states and defines the abstract MDP.\footnote{Although the definition is for infinite-horizon problems (due to $\gamma$), our finite-horizon problem readily converts to an infinite-horizon problem by adding an absorbing state that is reached after $t_\mathit{max}$ overt stages. The weight $w(s)$ is the probability the underlying state is $s$ given its abstract state $\phi(s)$ is observed. Any fixed policy, e.g. $\epsilon$-greedy with fixed $\epsilon$, induces a valid $w(s)$ and satisfies the definition.}

\begin{theorem}[Convergence of Q-learning under $Q^*$-irrelevance \cite{Li2006}]
Assume that each state-action pair is visited infinitely often and the step-size parameters decay appropriately. Q-learning with abstraction $\phi_{Q^*}$ converges to the optimal state-action value function in the ground MDP. Therefore, the resulting optimal abstract policy is also optimal in the ground MDP.
\end{theorem}

Because Li et al. do not consider action abstractions, we redefine the ground MDP to have the same actions as sense-move-effect MDP. Additionally, to keep the ground MDP Markov, we add the current level $l$, and the current point of focus $v \in \{1,\dots,m\}^3$, to the state. This does not essentially change the tabular pegs on disks domain but merely allows us to rigorously make the following connection.

Let states and actions of the ground MDP be denoted by $s$ and $a$ respectively. Similarly, let states and actions of the sense-move-effect MDP be denoted by $\bar s$ and $\bar a$ respectively. Let $\phi_\mathit{SME}: S \to \bar S$ be the observation function.

\begin{theorem}[$\phi_\mathit{SME}$ is $Q^*$-irrelevant] The sense-move-effect abstraction, $\phi_\mathit{SME}$, is a $Q^*$-irrelevance abstraction.
\label{th:qStarIrrelevance}
\end{theorem}

\begin{proof}
$Q^*(s,a)$ can be computed from $\bar s$ and $\bar a$. The reward after the current overt stage $t$ depends on $h$, whether or not it is possible to select a peg/disk, and whether or not it is possible to avoid selecting a placed peg. These are known from $\bar s$ and $\bar a$. Furthermore, whether or not a peg will be held after the current stage can be determined from $\bar s$ and $\bar a$. Finally, due to $t_\mathit{max}=2n$ and the fact that all pegs are initially unplaced, the sum of future rewards following an optimal policy from the current stage depends only on (a) whether or not a peg will be held after the current stage and (b) the amount of time left, $t-1$.
\end{proof}

\subsubsection{Simulation Results}
\label{sec:tabularSimulationResults}

% Experiments overview.

In these experiments, there were $n = 3$ objects, and the grid size was $m = 16$. Besides Deictic Image Mapping (where $L=1$), the number of levels was $L=4$. A comparison with no abstraction or HSA with $L=1$ was not possible because the system quickly ran out of memory (Eq.~\ref{eq:tabularGroundMdp}). The learning algorithm was Sarsa \cite{Rummery1994}, and actions were taken greedily w.r.t. the current Q-estimate. An optimistic initialization of action-values and random tie-breaking were relied on for exploration.

% With and without faulty sensor.

The proof to Theorem~\ref{th:qStarIrrelevance} suggests the observability of pegs, disks, placed pegs, and empty space are all important for learning the optimal policy. On the other hand, we empirically found no disadvantage to removing the $G_\mathit{pd}$ (placed pegs) and $G_e$ (empty space) grids. However, it is important to distinguish unplaced pegs and placed pegs. Fig.~\ref{fig:tabularPegsOnDisksFaultySensor} shows learning curves for an HSA agent with $G_p$ and $G_d$ grids versus an HSA agent with the same grids but showing pegs/disks irregardless of whether or not they are placed/occupied.

\begin{figure}[th]
  \centering
  \includegraphics[width=\linewidth]{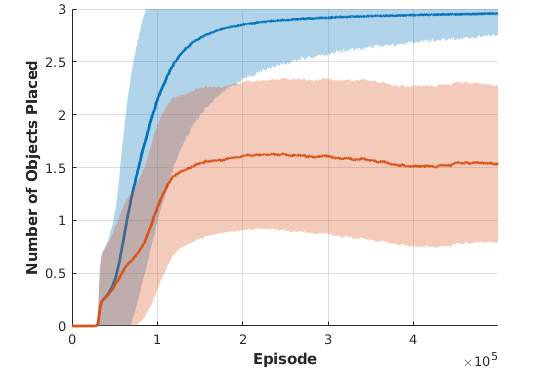}
  \caption{Number of objects placed for the standard HSA agent (blue) and a standard HSA agent with a faulty sensor (red). Curves are first mean and $\pm \sigma$ over each episode in 30 realizations, then averaged over $1,000$-epsisode segments for visualization.}
  \label{fig:tabularPegsOnDisksFaultySensor}
\end{figure}

% Standard HSA vs lookahead HSA vs deictic.

Lookahead HSA and Deictic Image Mapping variants (Section~\ref{sec:lookahead} and~\ref{sec:otherApproaches}) result in an even smaller state-action space than standard HSA. In the tabular domain, this means faster convergence (Fig.~\ref{fig:tabularPegsOnDisksStandardLookaheadDeictic}). Although the deictic representation seems superior in these results, it has a serious drawback. The action-selection time scales linearly with $m^3$ because there is one action for each cell in the underlying grid. The lookahead variant captures the best of both worlds -- small representation and fast execution. Thus, in the tabular domain, lookahead appears to be the satisfactory middle ground between the two approaches. However, for more complex domains, where Q-function approximation is required, the constant time needed to generate the action images becomes more significant, and the advantage of lookahead in terms of episodes to train diminishes (Section~\ref{sec:uprightPegsOnDisks}).

\begin{figure}[th]
  \centering
  \includegraphics[width=\linewidth]{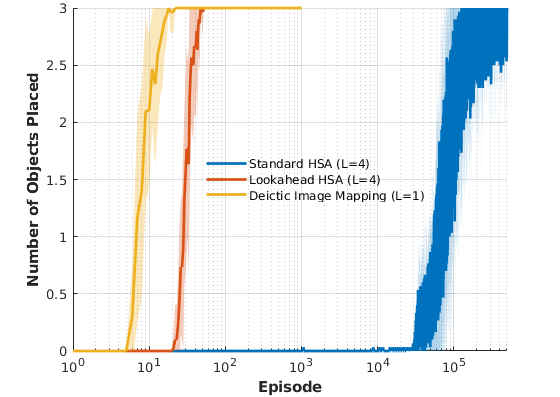}
  \caption{Number of objects placed for standard HSA (blue), lookahead HSA (red), and Deictic Image Mapping (yellow) agents. Curves are mean (solid) and $\pm \sigma$ (shaded) over 30 realizations. Plot in log scale for lookahead and deictic results to be visible.}
  \label{fig:tabularPegsOnDisksStandardLookaheadDeictic}
\end{figure}

% -----------------------------------------------------------------------------------------
\subsection{Upright Pegs on Disks}
\label{sec:uprightPegsOnDisks}

In this domain, pegs and disks are modeled as tall and flat cylinders, respectively, where the cylinder axis is always vertical (Fig.~\ref{fig:pegsOnDisksUpright}, left). Unlike the tabular domain, object size and position are sampled from a continuous space. Grasp and place success are checked with a set of simple conditions appropriate for upright cylinders.\footnote{Grasp conditions: gripper is collision-free and the top-center of exactly 1 cylinder is in the gripper's closing region. Place conditions: entire cylinder is above an unoccupied disk and the cylinder bottom is at most 1 cm below or 2 cm above the disk surface.} The reward is 1 for grasping an unplaced peg, -1 for grasping a placed peg, 1 for placing a peg on an unoccupied disk, and 0 otherwise.

Observations consist of 1 or 2 images ($k=2$, $n_\mathit{ch}=1$, $n_x=n_y=64$); the current HSA level, $l \in \{1,2,3\}$; and the e.e. status, $h \in \{\mathit{empty}, \mathit{holding}\}$. Each HSA level selects $(x,y,z)$ position (Fig.~\ref{fig:pegsOnDisksUpright}, right). Gripper orientation is not critical for this problem.

\begin{figure}[th]
  \centering
  \includegraphics[width=0.235\linewidth]{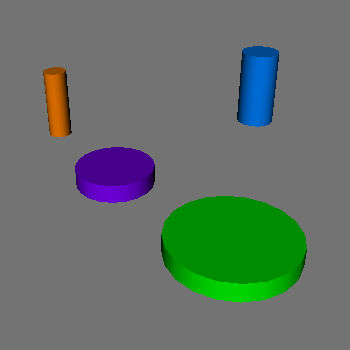}
  \includegraphics[width=0.235\linewidth]{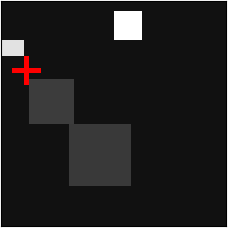}
  \includegraphics[width=0.235\linewidth]{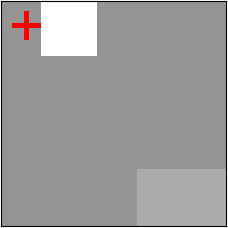}
  \includegraphics[width=0.235\linewidth]{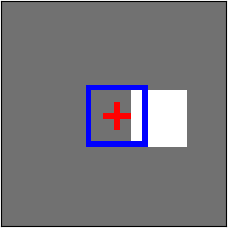}
  \caption{\textbf{Left}. Example upright pegs on disks scene. \textbf{Right}. Level 1, 2, and 3 images for grasping the orange peg. Red cross denotes the $(x,y)$ position selected by the robot and the blue rectangle denotes the allowed $(x,y)$ offset. $z_x=z_y=36$ cm$^2$ for level 1 and $9$ cm$^2$ for levels 2 and 3. $d_x=d_y=36$ cm$^2$ for level 1, $9$ cm$^2$ for level 2, and $2.25$ cm$^2$ for level 3. Pixel values normalized and height selection not shown for improved visualization.}
  \label{fig:pegsOnDisksUpright}
\end{figure}

\subsubsection{Network Architecture and Algorithm}

The Q-function consists of 6 convolutional neural networks (CNNs), 1 for each level and e.e. status, with identical architecture (Table~\ref{tab:uprightPegsOnDisksCnn}). This architecture results in faster execution time compared with our previous version \cite{Gualtieri2018B}. The loss is the squared difference between predicted and actual action-value target, averaged over a mini-batch. The action-value target is the reward received at the end of the current overt stage.%
\footnote{With standard MC and $\gamma=1$, the action-value target would be the sum of rewards received after the current time step \cite{Sutton2018}. Since, for this problem, no positively rewarding grasp precludes a positively rewarding place, ignoring rewards after the current overt stage is acceptable.} For CNN optimization, Adam \cite{Kingma2015} is used with a base learning rate of $0.0001$, weight decay of $0.0001$, and mini-batch size of 64.

\begin{table}[ht]
  \centering
  \begin{tabular}{|c|c|c|c|}
  \hline
  \textbf{layer} & \textbf{kernel size} & \textbf{stride} & \textbf{output size}\\
  \hline
  conv-1 & $7\times 7$ & 2 & $32\times 32\times 32$\\
  \hline
  conv-2 & $7\times 7$ & 2 & $16\times 16\times 64$\\
  \hline
  conv-3 & $7\times 7$ & 2 & $8\times 8\times 32$\\
  \hline
  conv-4 & $7\times 7$ & 2 & $4\times 4\times 32$\\
  \hline
  conv-5 & $7\times 7$ & 1 & $4\times 4\times 4$\\
  \hline
  \end{tabular}
  \caption{CNN architecture for the upright pegs on disks domain. Each layer besides conv-4 and conv-5 has a rectified linear unit (ReLU) as the activation.}
  \label{tab:uprightPegsOnDisksCnn}
\end{table}

\subsubsection{Simulation Results}

% 1 vs 2 vs 3 levels. 1 Level analogous to DQN.

We tested standard HSA with 1, 2, and 3 levels. The number of actions (CNN outputs) per level was adjusted so that each case had the same 5.625 mm precision in positioning of the e.e.: 1 level had $4^9$ outputs, 2 levels $4^3$ outputs and $4^6$ outputs, and 3 levels each had $4^3$ outputs. Note that with 1 level this is the DQN (i.e. no-hierarchy) approach (Section~\ref{sec:otherApproaches}). Exploration was $\epsilon$-greedy with $\epsilon = 0.04$.

Results are shown in Fig.~\ref{fig:pegsOnDisksUpright-123Levels}. The 1 level case is faster in terms of episodes because learning is over fewer time steps. The 2 levels case initially learns faster for the same reason. The 1 and 2 level cases converge to higher values because, with 3 levels, there is a higher chance of taking a random action during an overt stage. This is because more levels imply more time steps over which a random action could be selected w.p. $\epsilon$. What is important is that, in the last $5,000$ episodes when $\epsilon=0$, all scenarios have similar performance. However, HSA trains faster than DQN in terms of wall clock time (1.29 versus 2.55 hours) because fewer actions need evaluated ($192$ versus $262,144$). This advantage becomes more staggering as dimensionality of the action space increases, as in following sections.

\begin{figure}[th]
  \centering
  \includegraphics[width=\linewidth]{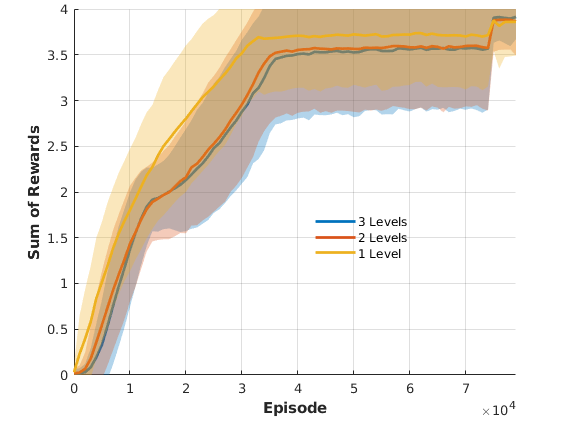}
  \caption{Standard HSA with varying number of levels. (Blue) $L=3$, (red) $L=2$, and (yellow) $L=1$. Curves are mean $\pm\sigma$ over 10 realizations then averaged over $1,000$ episode segments.}
  \label{fig:pegsOnDisksUpright-123Levels}
\end{figure}

% Choice of HSA values.

In another experiment we tested the sensitivity of standard HSA to the choice of $z$ and $d$ parameters. As explained in Section~\ref{sec:parameters}, these parameters are selected based on task geometry. If $z$ (resp. $d$) is too small, parts of the workspace will not be perceivable (resp. reachable). On the other hand, if $z$ is too large, the scene will not be visible in detail (because the perceived images are of fixed resolution), and if $d$ is too large, the samples at the last level will not be dense, resulting in low e.e. precision. Results for different values of $z$ and $d$ are shown in Table~\ref{tab:hsaParameters}. The ``ideal'' values are those selected according to the principles in Section~\ref{sec:parameters} and correspond to the 3-levels case in Fig.~\ref{fig:pegsOnDisksUpright-123Levels}. As expected, performance is much worse when selecting $z$ and $d$ without consideration to task geometry.

\begin{table}[ht]
  \centering
  \begin{tabular}{|c|c|c|c|}
  \hline
  & \textbf{Small} & \textbf{Ideal} & \textbf{Large}\\
  \hline
  level-1, $z_{xy} =$ & 36.0 & 36.0 & 36.0\\
  \hline
  level-1, $d_{xy} =$ & 36.0 & 36.0 & 36.0\\
  \hline
  level-2, $z_{xy} =$ & 6.00 & 9.00 & 12.00\\
  \hline
  level-2, $d_{xy} = $ & 6.00 & 9.00 & 12.00\\
  \hline
  levle-3, $z_{xy} = $ & 6.00 & 9.00 & 12.00\\
  \hline
  level-3, $d_{xy} = $ & 1.50 & 2.25 & 3.00\\
  \hline
  $\mu$ Return & 2.69 & 3.91 & 2.83\\
  \hline
  $\sigma$ Return & 1.32 & 0.01 & 1.75\\
  \hline
  \end{tabular}
  \caption{Varying standard HSA parameters $z_{xy}$ and $d_{xy}$ (in cm). ``Ideal'' values were selected according to Section~\ref{sec:parameters}. ``Small'' (resp. ``Large'') values are smaller (resp. larger) than ideal. Last 2 rows are average and standard deviation over sum of rewards per episode, after 10 different training sessions and $1,000$ episodes per session.}
  \label{tab:hsaParameters}
\end{table}

% Lookahead vs standard. Why we didn't do deictic comparison.

We also compared standard HSA to lookahead HSA, both with 3 levels. We did not compare to the Deictic Image Mapping approach (Lookahead HSA with 1 level) because computation of all $4^9$ images was prohibitively expensive. Results are shown in Fig.~\ref{fig:pegsOnDisksUpright-standardVsLookahead}. In contrast to the tabular results, both scenarios perform similarly. We hypothesize that the advantage of lookahead HSA is lost due to the equivariance property of CNNs. Since execution time for standard HSA is less than half that of lookahead (1.29 versus 3.67 hours), from now on we only consider standard HSA.

\begin{figure}[th]
  \centering
  \includegraphics[width=\linewidth]{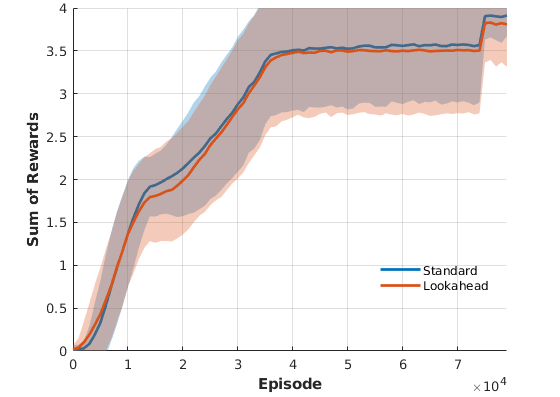}
  \caption{Standard HSA (blue) versus lookahead HSA (red).}
  \label{fig:pegsOnDisksUpright-standardVsLookahead}
\end{figure}

% -----------------------------------------------------------------------------------------
\subsection{Bottles on Coasters}
\label{sec:bottlesOnCoasters}

The main question addressed here is if HSA can be applied to a practical problem and implemented on a physical robotic system. The bottles on coasters domain is similar to the pegs on disks domain, but now objects have complex shapes and are required to be placed upright.\footnote{Grasp conditions: gripper closing region intersects exactly 1 object and the antipodal condition from \cite{tenPas2017} with $15^\circ$ friction cone. Place conditions: bottle is upright, center of mass (CoM) $(x,y)$ position at least 2 cm inside an unoccupied coaster, and bottom within $\pm 2$ cm of coaster surface.} The reward is $1$ for grasping an unplaced object more than 4 cm from the bottom (placing with bottom grasps is kinematically infeasible in the physical system), $-1$ for grasping a placed object, $1$ for placing a bottle, and 0 otherwise.

Observations are similar to before except now the image resolution is lower ($n_x=n_y=48$), and the overt time step is always input to grasp networks (and never input to place networks). HSA has 3 levels selecting $(x,y,z)$ position and 1 level selecting orientation about the gripper approach axis (Fig.~\ref{fig:hsaExample}).

To achieve the target precision in e.e. pose (3.75 mm position and $6^\circ$ orientation for grasping), DQN (or 1-level HSA) would need to evaluate over 53 million actions. Evaluation was prohibitively expensive with our computing hardware. HSA only needs 404 actions (although we use 708 to achieve redundancy, with little loss in computation time as the evaluation is done in parallel).

\subsubsection{Network Architecture and Algorithm}

The network architecture is shown in Table~\ref{tab:bottlesOnCoastersCnn}. There is 1 network for each HSA level and e.e. status. Weight decay is 0. Q-network targets are the reward after the current overt stage.

\begin{table}[ht]
  \centering
  \begin{tabular}{|c|c|c|c|}
  \hline
  \textbf{layer} & \textbf{kernel size} & \textbf{stride} & \textbf{output size}\\
  \hline
  conv-1 & $8\times 8$ & 2 & $24\times 24\times 64$\\
  \hline
  conv-2 & $4\times 4$ & 2 & $12\times 12\times 64$\\
  \hline
  conv-3 & $3\times 3$ & 2 & $6\times 6\times 64$\\
  \hline
  conv-4 / ip-1 & $2\times 2$ / - & 1 / - & $6^3 / n_\mathit{orient}$\\
  \hline
  \end{tabular}
  \caption{CNN architecture for the bottles on coasters domain. Each layer besides the last has a ReLU activation. The last layer is a convolution layer for levels 1-3 (selecting position) and an inner product (IP) layer for level 4 (selecting orientation). $n_\mathit{orient} = 60$ for grasp networks and $n_\mathit{orient} = 3$ for place networks.}
  \label{tab:bottlesOnCoastersCnn}
\end{table}

\subsubsection{Simulation Results}

70 bottles from 3DNet \cite{Wohlkinger2012} were randomly scaled to height 10-20 cm. Bottles were placed upright with probability $1/3$ and on their sides with probability $2/3$. Learning curves for 2 bottles and 2 coasters are shown in Fig.~\ref{fig:bottlesOnCoasters-simResult}. Performance is lower than that of the upright pegs on disks domain, reflective of the additional problem complexity.

\begin{figure}[th]
  \centering
  \includegraphics[width=\linewidth]{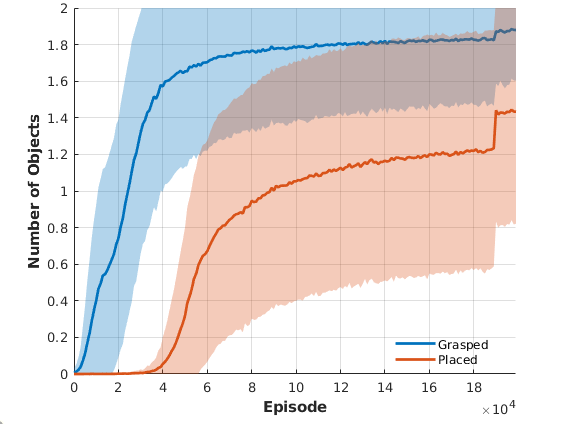}
  \caption{Number of bottles grasped (blue) and placed (red). Curves are mean $\pm\sigma$ over 10 realizations then averaged over $1,000$ episode segments.  Standard HSA with $L=4$.}
  \label{fig:bottlesOnCoasters-simResult}
\end{figure}

To test robustness of the system to background noise, we ran the same experiment with the addition of distractor objects. These distractors are 3 rectangular blocks, with side lengths 1 to 4 cm, scattered randomly in the scene (e.g., Fig.~\ref{fig:bottlesOnCoasters-clutter}, left). Learning performance is only slightly lower (Fig.~\ref{fig:bottlesOnCoasters-clutter}, right). However, if clutter is present at test time, it is important to train the system with clutter. The robot trained without clutter places an average of 1.24 bottles in the cluttered environment (versus 1.55 if trained with clutter). The distractors are visible at some levels (e.g., level 1), so the robot does need to learn to ignore (and avoid collisions with) them.

\begin{figure}[H]
  \centering
  \includegraphics[height=1.4in]{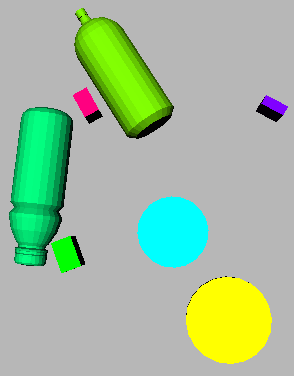}
  \includegraphics[height=1.4in]{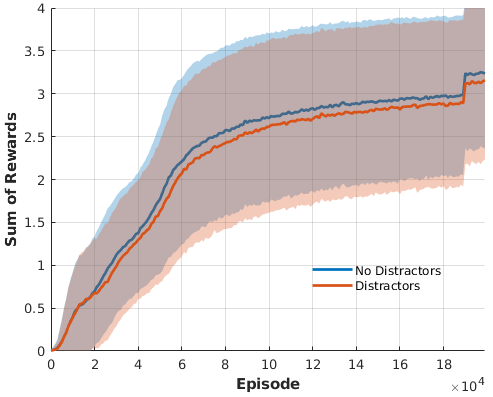}
  \caption{\textbf{Left.} Scene with clutter. \textbf{Right.} Learning curves comparing average sum of rewards when distractors are not present (blue) and present (red).}
  \label{fig:bottlesOnCoasters-clutter}
\end{figure}

\subsubsection{Top-$n$ Sampling}

Before considering experiments on a physical robotic system, we address an important assumption of the move-effect system of Section~\ref{sec:problem}. The assumption is the e.e. can move to any pose, $T_\mathit{ee}$, in the robot's workspace. Recent advances in motion planning algorithms make this a reasonable assumption for the most part; nonetheless, a pose can still sometimes be unreachable due to obstacles, motion planning failure, or IK failure. 

To address this issue, multiple, high-valued actions are sampled from the policy learned in simulation. In particular, for each level $l$ of an overt stage, we take the top-$n$ samples according to Eq.~\ref{eq:successProbability}, where $Q_l$ is the action-value estimate at level $l$, $Q_\mathit{max}$ is the maximum possible action-value, $Q_\mathit{min}$ is the minimum possible action-value, and $p_0 = 1$.

\begin{align}
p_l &= p_{l-1}  \frac{Q_l - Q_\mathit{min}}{Q_\mathit{max} - Q_\mathit{min}}, &l=1, \dots, L
\label{eq:successProbability}
\end{align}

Preliminary tests in simulation showed sampling top-$n$ $p_l$ values performs better than sampling top-$n$ $Q_L$ values, as was done previously \cite{Gualtieri2018B}. Sampling top-$n$ $p_l$ values may be viewed as an ensemble method where each level votes on the final overt action (cf. \cite{Anschel2017}).

During test time, the resulting $n$, $T_\mathit{ee}$ samples are checked for IK and motion plan solution in descending order of $p_L$ value. As $n$ increases, the probability of failing to find a reachable e.e. pose decreases; however, the more poses that are unreachable, the lower the $p_L$ value. Thus, when designing an HSA system, it is important to not over constrain the space of actions.

\subsubsection{Robot Experiments}

We tested the bottles on coasters task with the physical system depicted in Fig.~\ref{fig:system}. The system consists of a Universal Robots 5 (UR5) arm, a Robotiq 85 parallel-jaw gripper, and a Structure depth sensor. The test objects (Fig.\ref{fig:objects}) were not observed during training. The CNN weight files had about average performance out of the 10 realizations (Fig.~\ref{fig:bottlesOnCoasters-simResult}).

\begin{figure}[ht]
  \centering
  \includegraphics[width=0.70\linewidth]{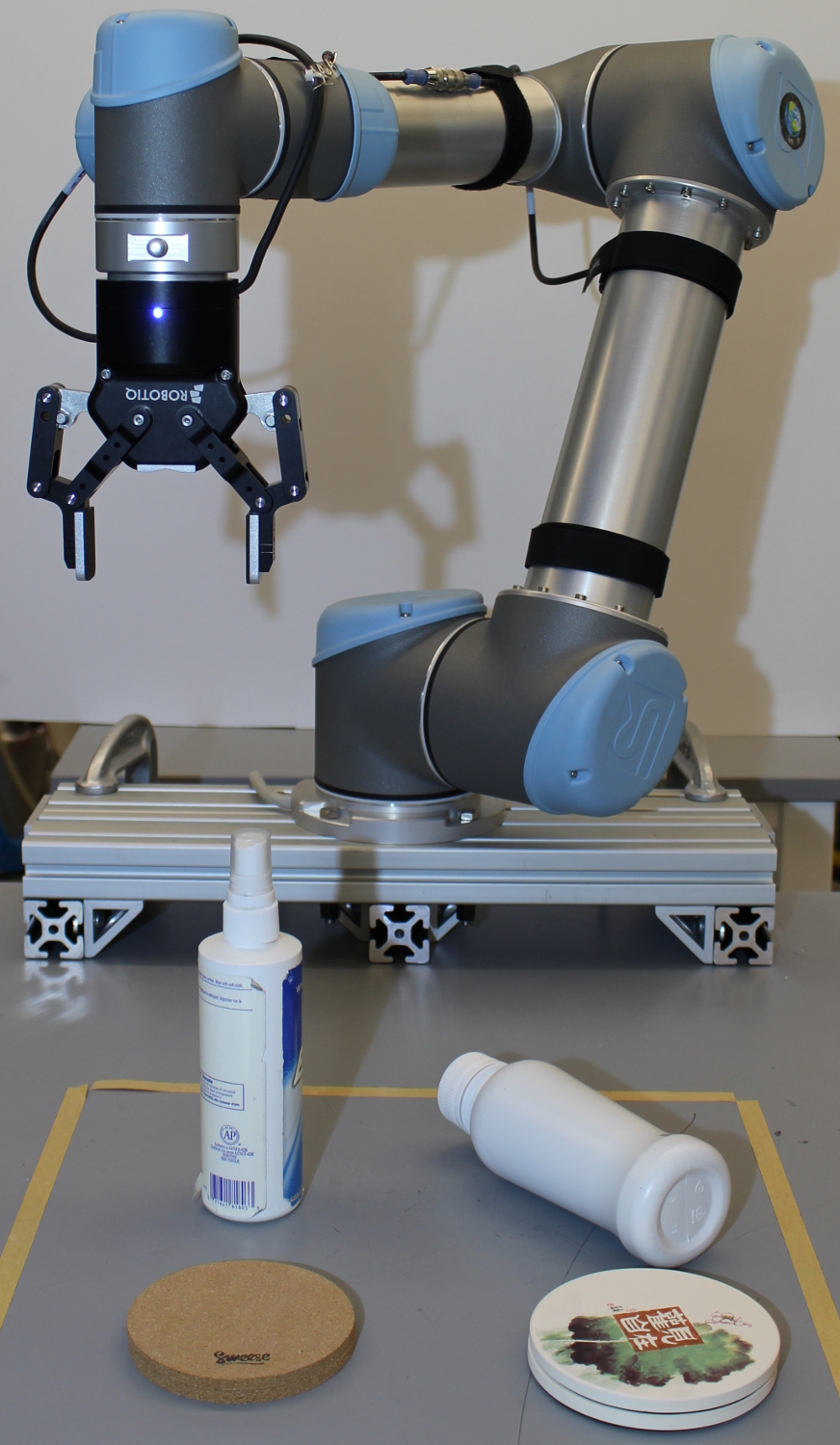}
  \caption{Test setup for bottles on coasters task: a UR5 arm, Robotiq 85 gripper, Structure depth sensor (mounted out of view above the table and looking down), 2 bottles, and 2 coasters.}
  \label{fig:system}
\end{figure}

Initially, 2 coasters were randomly selected and placed in arbitrary positions in the back half of the robot's workspace (too close resulted in unreachable places). Then, 2 bottles were randomly selected and placed upright with probability $1/3$ and on the side with probability $2/3$. The bottles were not allowed to be placed over a coaster.%
\footnote{Python's pseudorandom number generator was used to decide the objects used and upright/side placement. Object position was decided by a human instructed to make the scenes diverse.}
Top-$n$ sampling with $n=200$ was used. A threshold was set for the final grasp/place approach, whereby, if the magnitude of the force on the arm exceeded this threshold, the motion canceled and the open/close action was immediately performed.

\begin{figure}[ht]
  \centering
  \includegraphics[width=0.70\linewidth]{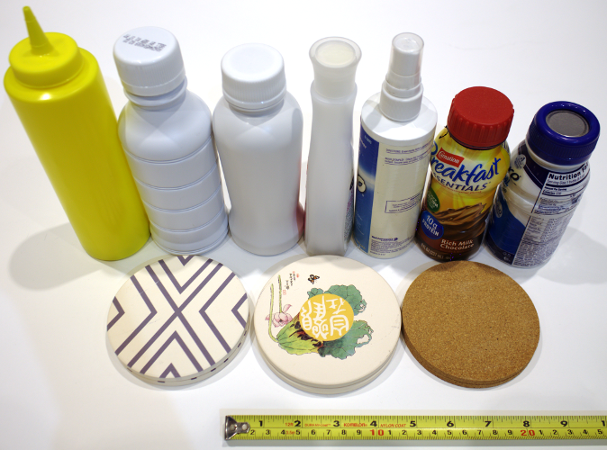}
  \caption{Test objects used in UR5 experiments.}
  \label{fig:objects}
\end{figure}

Results are summarized in Table~\ref{tab:ur5Results}, and a successful sequence is depicted in Fig.~\ref{fig:ur5Episode}. A grasp was considered successful if a bottle was lifted to the ``home'' configuration; a place was considered successful if a bottle was placed upright on an unoccupied coaster and remained there after the gripper withdrew. Failures were: grasped a placed object ($\times3$), placed too close to the edge of a coaster and fell over ($\times 3$), placed upside-down ($\times2$), object slip in hand after grasp caused a place failure ($\times1$), and object fell out of hand after grasp ($\times1$).

\begin{figure*}[ht]
  \centering
  \includegraphics[width=0.24\linewidth]{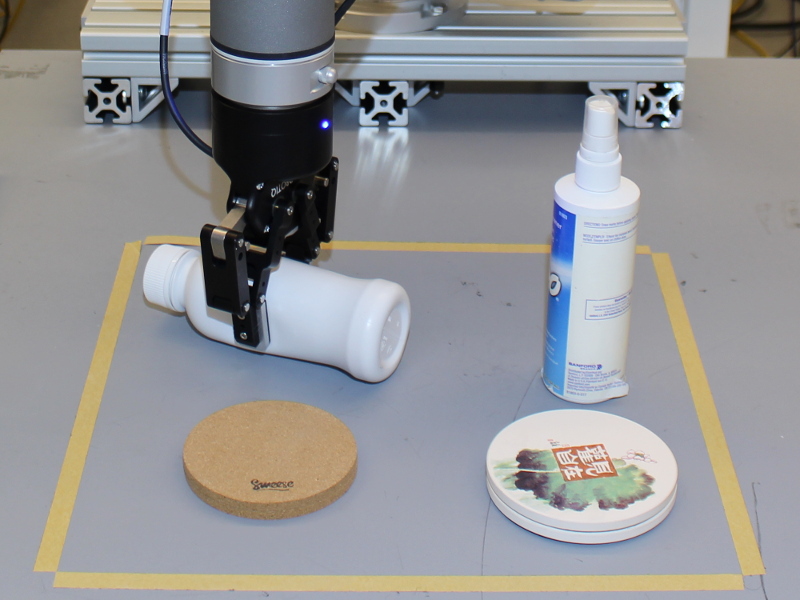}
  \includegraphics[width=0.24\linewidth]{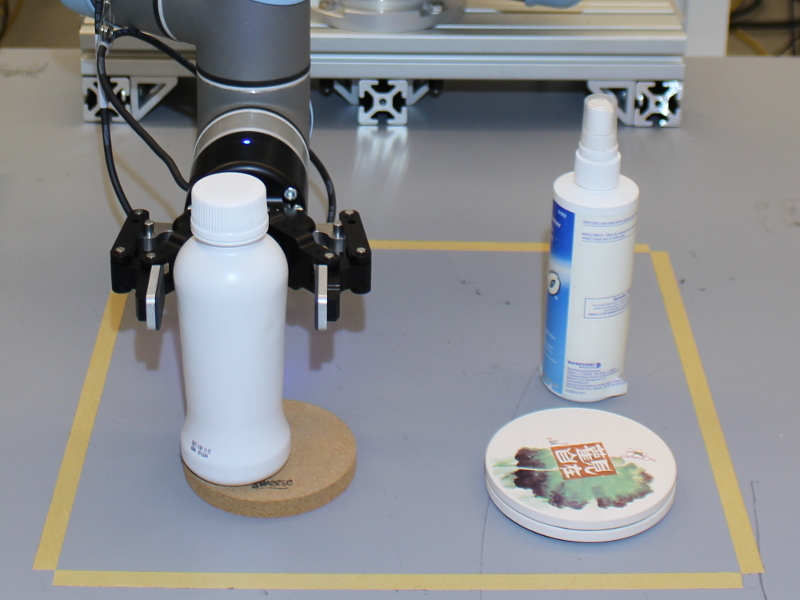}
  \includegraphics[width=0.24\linewidth]{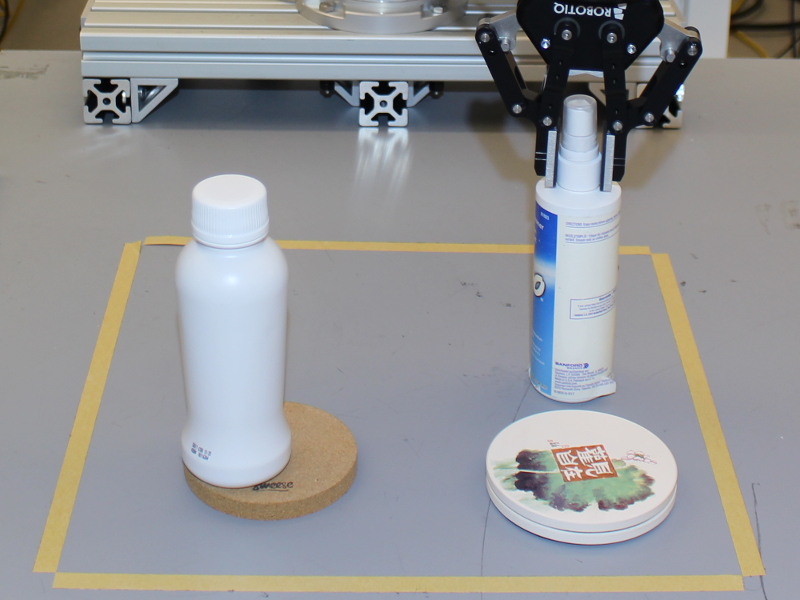}
  \includegraphics[width=0.24\linewidth]{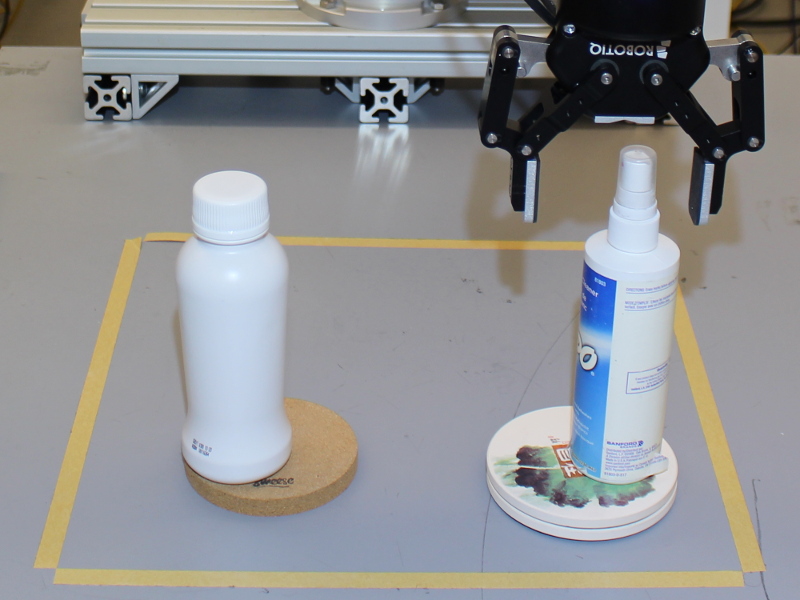}
  \caption{Successful trial -- all bottles placed in 4 overt stages. Image taken immediately after open/close.}
  \label{fig:ur5Episode}
\end{figure*}

\begin{table}[ht]
  \centering
  \begin{tabular}{|c|c|c|}
  \hline
  & Grasp & Place\\
  \hline
  Attempts & 60 & 59\\
  \hline
  Success Rate & 0.98 & 0.90\\
  \hline
  Number of Objects & $1.97 \pm 0.18$ & $1.67 \pm 0.48$ \\
  \hline
  \end{tabular}
  \caption{Performance for UR5 experiments placing 2 bottles on 2 coasters averaged over 30 episodes with $\pm \sigma$. Task success rate with $t_\mathit{max}=4$ was 0.67.}
  \label{tab:ur5Results}
\end{table}

% -----------------------------------------------------------------------------------------
\subsection{6-DoF Pick-Place}
\label{sec:pickPlace}

The HSA method was also implemented for 6-DoF manipulation, and the same system was tested on 3 different pick-place tasks \cite{Gualtieri2018B}.%
\footnote{This section refers to an earlier version of our system, so the simulations took longer and the success rates for bottles are lower. The setup was similar to that in Fig.~\ref{fig:system} except the sensor was mounted to the wrist. See \cite{Gualtieri2018B} for more details.}
The tasks included stacking a block on top of another, placing a mug upright on the table, and (similar to Section~\ref{sec:bottlesOnCoasters}) placing a bottle on a coaster. All tasks included novel objects in light to moderate clutter (Fig.~\ref{fig:6DofPickPlace}). To handle perceptual ambiguities in mugs, the observations were 3-channel images ($k=2$, $n_\mathit{ch}=3$, $n_x=n_y=60$) projected from a point cloud obtained from 2 camera poses. HSA included 6 levels ($L=6$) -- 3 for $(x,y,z)$ position and 1 for each Euler angle. Results from UR5 experiments are shown in Table~\ref{tab:pickPlace}.

\begin{figure}[ht]
  \centering
  \includegraphics[width=0.49\linewidth]{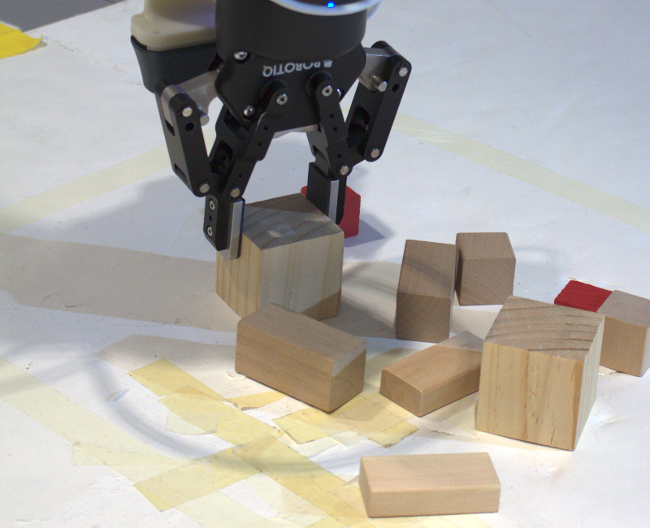}
  \includegraphics[width=0.49\linewidth]{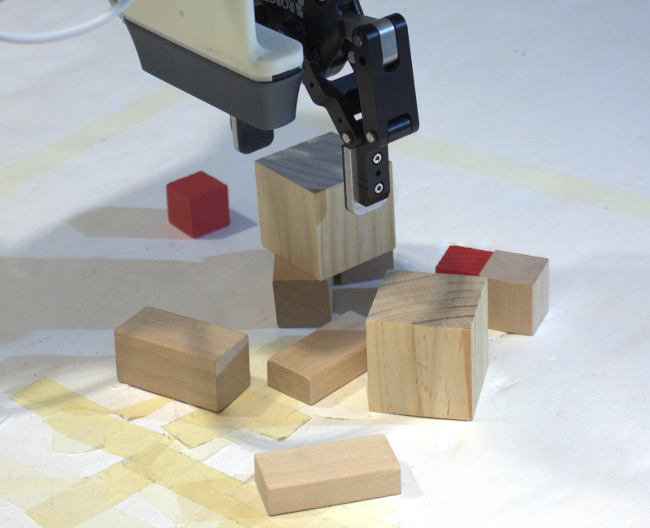}\\
  \smallskip
  \includegraphics[width=0.49\linewidth]{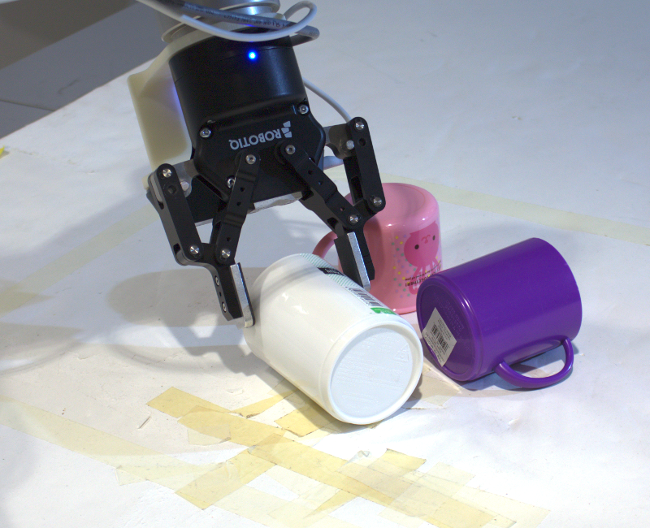}
  \includegraphics[width=0.49\linewidth]{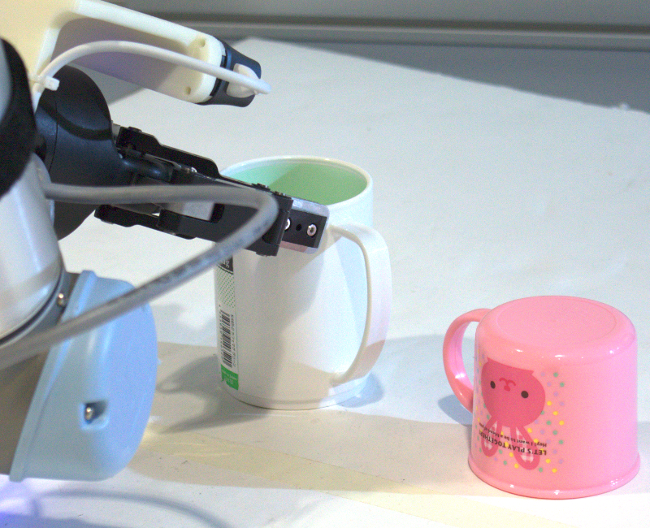}
  \caption{6-DoF pick place on the UR5 system. \textbf{Top}. Blocks task. \textbf{Bottom}. Mugs task. Notice the grasp is diagonal to the mug axis, and the robot compensates for this by placing diagonally with respect to the table surface.}
  \label{fig:6DofPickPlace}
\end{figure}

\begin{table}[ht]
  \centering
  \begin{tabular}{|c|c|c|c|} 
  \hline
  & Blocks & Mugs & Bottles\\
  \hline
  Grasp & 0.96 & 0.86 & 0.89\\
  \hline
  Place & 0.67 & 0.89 & 0.64\\
  \hline
  Task & 0.64 & 0.76 & 0.57\\
  \hline
  \hline
  $n$ Grasps & 50 & 51 & 53\\
  \hline
  $n$ Places & 48 & 44 & 47\\
  \hline
  \end{tabular}
  \caption{\textbf{Top.} Grasp, place, and task success rates for the 3 tasks with $t_\mathit{max}=2$ (i.e., 1 pick 1 place). \textbf{Bottom.} Number of grasp and place attempts.}
  \label{tab:pickPlace}
  \vspace{-0.2in}
\end{table}

% =========================================================================================
\section{Conclusion}
\label{sec:conclusion}

The primary conclusion is that the sense-move-effect abstraction, when coupled with hierarchical spatial attention, is an effective way of simultaneously handling (a) high-resolution 3D observations and (b) high-dimensional, continuous action spaces. These two issues are intrinsic to realistic problems of robot learning. We provide several other considerations relevant to systems employing spatial attention:

\subsection{Secondary Conclusions}

\begin{itemize}[leftmargin=*]
\item Compared to a flat representation, HSA can result in an exponential reduction in the number of actions that need to be sampled (Section~\ref{sec:hsa}).
\item HSA generalizes DQN, and lookahead HSA generalizes Deictic Image Mapping (Section~\ref{sec:otherApproaches}).
\item The partial observability induced by an HSA observation does not preclude learning an optimal policy (Section~\ref{sec:tabularDomain}).
\item HSA may take longer to learn than DQN in terms of the number of episodes to convergence, but HSA executes faster when the number of actions is large (Section~\ref{sec:uprightPegsOnDisks}).
\item Lookahead HSA is preferred to standard HSA in terms of the number of the episodes to train, but execution time is longer by a constant and the learning benefit diminishes when coupled with function approximation (Sections~\ref{sec:tabularDomain}~and~\ref{sec:uprightPegsOnDisks}).
\item HSA can be applied to realistic problems on a physical robotic system (Sections~\ref{sec:bottlesOnCoasters}~and~\ref{sec:pickPlace}).
\end{itemize}

\subsection{Limitations and Future Work}

A concern with all deep RL methods is that modeling and optimization errors induced by the use of function approximation prevent the robot from learning an optimal policy. This is true for even simple problems, such as the upright pegs on disks problem of Section~\ref{sec:uprightPegsOnDisks}. Also, how manipulation skills can be automatically and efficiently transferred to different but related tasks remains an open question. Even small changes to the task, such as the inclusion of distractor objects, requires complete retraining of the system for maximum performance. Finally, HSA is a competing approach to policy search methods in that both can handle high-dimensional, continuous action spaces. It would be interesting to see a comparison between these approaches. Previous value-based approaches like DQN cannot handle the high-dimensional action spaces prevalent in robotics; thus, HSA enables a comparison between value and policy search methods for these domains.

% =========================================================================================
\section*{Acknowledgment}

We thank Andreas ten Pas for reviewing early drafts of this paper and the anonymous reviewers for their insightful comments. Funding was provided by NSF (IIS-1724257, IIS-1724191, IIS-1750649, IIS-1830425, IIS-1763878), ONR (N00014-19-1-2131), and NASA (80NSSC19K1474).

% =========================================================================================
\bibliographystyle{IEEEtran}
\bibliography{References}

% =========================================================================================

% biography section

\begin{IEEEbiography}[{\includegraphics[width=1in,height=1.25in,clip,keepaspectratio]{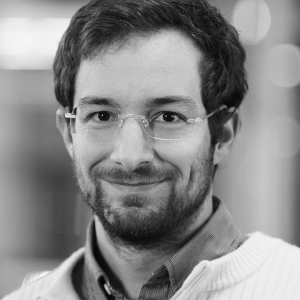}}]{Marcus Gualtieri} is a PhD student at Northeastern University in Boston, Massachusetts. In 2017 he received the MS degree in computer science from Northeastern, and 2008 he received the BS degree in software engineering from Florida Institute of Technology. His research interests include robot learning and planning in unstructured environments.
\end{IEEEbiography}

\begin{IEEEbiography}[{\includegraphics[width=1in,height=1.25in,clip,keepaspectratio]{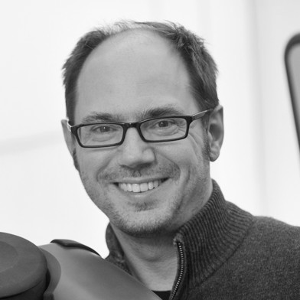}}]{Robert Platt} is an associate professor at Northeastern. Prior to that, he was a research scientist at MIT and a robotics engineer at NASA. He earned his PhD in Computer Science in 2006 from the University of Massachusetts, Amherst. His research interests primarily include perception, planning, and control for robotic manipulation.
\end{IEEEbiography}

% You can push biographies down or up by placing a \vfill before or after them. The appropriate use of \vfill depends on what kind of text is on the last page and whether or not the columns are being equalized.

\vfill

% Can be used to pull up biographies so that the bottom of the last one is flush with the other column.
%\enlargethispage{-5in}

\end{document}